\documentclass[times, 10pt,twocolumn]{article} 
\usepackage{latex8}
\usepackage{times}

\usepackage{amsmath}
\usepackage{amssymb}
\usepackage{amsthm}
\usepackage{booktabs}
\usepackage{url}
\usepackage{graphicx}

\newcommand{\pr}[1]{\left(#1\right)}
\newcommand{\spr}[1]{\left[#1\right]}
\newcommand{\abs}[1]{\left|#1\right|}

\newcommand{\prob}[1]{\mathbb{P}\pr{#1}}

\newcommand{\mean}[1]{\operatorname{E}\spr{#1}}
\newcommand{\var}[1]{\operatorname{Var}\spr{#1}}

\newcommand{\ncd}{\mathrm{ncd}}

\newcommand{\cd}{\mathrm{cd}}
\newcommand{\ind}[1]{\mathrm{ind}\pr{#1}}

\newcommand{\proj}[2]{\pi_{#1}\pr{#2}}

\newcommand{\corr}[1]{\mathrm{corr}\pr{#1}}

\newcommand{\funcdef}[3]{{#1}:{#2} \to {#3}}

\newcommand{\cdA}[2]{\cd_A\pr{{#1} ; {#2}}}
\newcommand{\cdR}[2]{\cd_R\pr{{#1} ; {#2}}}
\newcommand{\cdAnoarg}[1]{\cd_A\pr{{#1}}}

\newcommand{\ncdA}[2]{\ncd_A\pr{{#1} ; {#2}}}

\newcommand{\ncdAnoarg}[1]{\ncd_A\pr{{#1}}}

\newcommand{\dtname}[1]{\emph{#1}}

\newtheorem{theorem}{Theorem}
\newtheorem{conjecture}[theorem]{Conjecture}
\newtheorem{example}[theorem]{Example}
\newtheorem{definition}[theorem]{Definition}
\newtheorem{proposition}[theorem]{Proposition}
\newtheorem{corollary}[theorem]{Corollary}

\title{What is the dimension of your binary data?}
\author{
Nikolaj Tatti \hspace{1cm} Taneli Mielik{\"a}inen \hspace{1cm} Aristides Gionis \hspace{1cm} Heikki Mannila \\
HIIT Basic Research Unit, Department of Computer Science \\
University of Helsinki and Helsinki University of Technology}

\begin{document}
\maketitle

\begin{abstract}
Many 0/1 datasets have a very large number of variables; on the other
hand, they are sparse and the dependency structure of the variables is
simpler than the number of variables would suggest. Defining the
effective dimensionality of such a dataset is a nontrivial problem. We
consider the problem of defining a robust measure of dimension for 0/1
datasets, and show that the basic idea of fractal dimension can be
adapted for binary data. However, as such the fractal dimension is
difficult to interpret. Hence we introduce the concept of normalized
fractal dimension. For a dataset $D$, its normalized fractal dimension
is the number of columns in a dataset $D'$ with independent columns
and having the same (unnormalized) fractal dimension as $D$.  The
normalized fractal dimension measures the degree of dependency
structure of the data. We study the properties of the normalized
fractal dimension and discuss its computation. We give empirical
results on the normalized fractal dimension, comparing it against
baseline measures such as PCA. We also study the relationship of the
dimension of the whole dataset and the dimensions of subgroups formed
by clustering. The results indicate interesting differences between
and within datasets.
\end{abstract}

\section{Introduction}
Many 0/1-datasets occurring in data mining are on one hand complex, as
they have a very high number of columns.  On the other hand, the
datasets can be simple, as they might be very sparse or have lots of
structure.
In this paper we consider the problem of defining a notion of {\em
effective dimension} for a binary dataset.  We study ways of defining
a concept of dimension that would somehow capture the complexity or
simplicity of the dataset.  Such a notion of effective dimension can
be used as a general score describing the complexity or simplicity of
the dataset; Some potential applications of the intrinsic
dimensionality of a dataset include model selection problems in data
analysis; it can also be used in speeding up certain computations
(see, e.g.,~\cite{faloutsos94beyond}).

For continuous data there are many ways of defining the dimension of a
dataset.  One approach is to use decomposition methods such as SVD,
PCA, or NMF (nonnegative matrix
factorization)~\cite{jolliffe02pca,lee00nmf} and to count how many
components are needed to express, say, $90\%$ of the variance in the
data.  This number of components can be viewed as the number of
effective dimensions in the data.

In the aforementioned methods it is assumed that the dataset is
embedded into a higher-dimensional space by some (smooth) mapping.
The other main approach is to use a different concept, that of fractal
dimensions~\cite{barnsley88fractals,faloutsos94beyond,kegl02intristic,ott97chaos}.
Very roughly, the concept of fractal dimension is based on the idea of
counting the number of observations in a ball of radius $r$ and
looking what the rate of growth of the number is as a function of $r$.
If the number grows as $r^k$, then the dimensionality of the data can
be considered to be $k$. Note that this approach does not provide any
mapping that can be used for the dimension reduction. Such mapping
does not even make sense because the dimension can be non-integral.

Applying these approaches to binary data is not easy.  Many of the
component methods, such as PCA and SVD are strongly based on the
assumption that the data are real-valued.  NMF looks for a matrix
decomposition with nonnegative entries and hence is somewhat better
suited for binary data.  However, the factor matrices may have
continuous values, which makes them difficult to interpret.  The
component techniques aimed at discrete data (such as multinomial
PCA~\cite{buntine03mpca} or latent Dirichlet allocation
(LDA)~\cite{blei03lda}) are possible alternatives, but interpreting
the results is hard.

In this paper we explore the notion of effective dimension for binary
datasets by using the basic ideas from fractal dimensions.
Essentially, we consider the distribution of the pairwise distances
between random points in the dataset.  Denoting by $Z$ this random
variable, we study the ratio of $\log \prob{Z < r}$ and $\log r$, for
different values of the $r$, and fit a straight line to this; the
slope of the line is the correlation dimension of the dataset.

Interpreting the correlation dimension of discrete data turns out to
be quite difficult too, because the values of the correlation
dimension tend to very small.  To relieve this problem, we normalize
them by considering what would be the number of variables in a dataset
with the same correlation dimension but with independent columns.
This {\em normalized correlation dimension} is our main concept.

We study the behavior of the correlation dimension and the normalized
correlation dimension, both theoretically and empirically.  We give
approximations for correlation dimension, in the case of independent
variables, showing that it decreases when the data becomes more
sparse.  We also give theoretical evidence indicating that positive
correlations between the variables lead to smaller correlation
dimensions.

Our empirical results for generated data show that the normalized
correlation dimension of a dataset with $K$ independent variables is
very close to $K$, irrespective of the sparsity of the attributes.  We
demonstrate that adding positive correlation decreases the dimension.
For real datasets, we show that different datasets have quite
different normalized correlation dimensions, and that the ratio of the
number of variables to the normalized correlation dimension varies a
lot.  This indicates that the amount of structure in the datasets is
highly variable.  We also compare the normalized correlation dimension
against the number of PCA components needed to explain $90\%$ of the
variance in the data, showing interesting differences among the
datasets.

The rest of this paper is organized as follows.  In
Section~\ref{sec:fractal} we define the correlation dimension for
binary datasets. we analyze the correlation dimension in
Section~\ref{sec:properties}.  The correlation dimension produces too
small values and hence in Section~\ref{sec:inverse} we provide means
for scaling the dimension. In Section~\ref{sec:experiments} we
represent our tests with real world datasets.  In
Section~\ref{sec:related} we review the related literature, and
Section~\ref{section:conclusions} is a short conclusion.
\section{Correlation Dimension}
\label{sec:fractal}

There are several possible definitions of the fractal dimension of a
subset of the Euclidean space; see,
e.g.,~\cite{barnsley88fractals,ott97chaos} for a survey; the {\em
R{\'e}nyi dimensions}~\cite{ott97chaos} form a fairly general family.
The standard definitions of the fractal dimension are not directly
applicable in the discrete case, but they can be modified to fit in.

The basic idea in the fractal dimensions is to study the distance
between two random data points. 

We focus on the correlation dimension.  Consider a 0/1 dataset $D$
with $K$ variables.  Denote by $Z_D$ the random variable whose value
is the $L_1$ distance between two randomly chosen points from $D$;
thus $0 \leq Z_D \leq K$.  Informally, the correlation dimension is
the slope of the line fitted in the log-log plot of $(r,\prob{Z_D <
r})$.

The more formal definition is more complex because the non-continuity
of $\prob{Z_D < r}$ causes misbehavior in our later definitions. To
remedy these problems we first define function
$\funcdef{f}{\mathbb{N}}{\mathbb{R}}$ to be $f\pr{r} = \prob{Z_D <
r}$. We extend this function to real numbers by linear
interpolation. Thus $f(r)$ is a continuous function being equal to
$\prob{Z_D < r}$ when $r$ is an integer.

Let $0 \leq r_1 < r_2 \leq K$.  Then the different radii $r$ and the
function $f$ for a given dataset $D$ determine the point set
\[
\begin{split}
\mathcal{I}\pr{D, r_1, r_2, N} = & \left\{\pr{\log r, \log f(r)} \mid \right.\\
& \quad \left. r = r_1 + \frac{i\pr{r_2 - r_1}}{N}, i = 0 \ldots N \right\}.
\end{split}
\]
We usually omit the parameter $N$ for the sake of brevity.

For example, assume that $\prob{Z_D \leq r} \propto r^d$ for some $d$,
that is, the number of pairs of points within distance $d$ grows as
$r^d$.  Then $\mathcal{I}(D, r_1, r_2)$ is a straight line and the
correlation dimension is equal to $d$.

\vspace{0.2cm}
\begin{definition} \label{def:cd}
The \emph{correlation dimension} $\cdR{D}{r_1, r_2}$ for a binary
dataset $D$ and radii $r_1$ and $r_2$ is the slope of the
least-squares linear approximation $\mathcal{I}\pr{Z, r_1, r_2}$.

Assume that we are given $\alpha_1$ and $\alpha_2$ such that $0 \leq
\alpha_1 < \alpha_2 \leq 1$.  We define $\cdA{D}{\alpha_1, \alpha_2}$
to be $\cdR{D}{r_1, r_2}$, where the radii $r_i$ are set to be
$\max\pr{f^{-1}\pr{\alpha_i}, 1}$. The reason for truncating $r_i$ is
to avoid some misbehavior occurring with extremely sparse datasets.
\end{definition}
\vspace{0.2cm}

That is, $\mathcal{I}\pr{D, r_1, r_2}$ is the set of points containing
the logarithm of the radius $r$ and the logarithm of the fraction of
pairs of points from $D$ that have $L_1$ distance less than or equal
to $r$.  The correlation dimension is the slope of the line that fits
these points best. The difference
between $\cdR{D}{r_1, r_2}$ and $\cdA{D}{\alpha_1, \alpha_2}$ is that
$\cd_R$ is defined by using the absolute bounds $r_1$ and $r_2$ for
the radius $r$, whereas $\cd_A$ uses the parameters $\alpha_1$ and
$\alpha_2$ to specify the sizes of the tail of the distribution.  For
instance, $\cdA{D}{1/4,3/4}$ is the correlation dimension obtained by
first computing the values $r_1$ and $r_2$ such that one quarter of
the pairs of points have distance below $r_1$, and one quarter of the
pairs have distance above $r_2$.  The dimension is then obtained by
computing $N + 1$ points $\pr{\log r, \log f(r)}$ with $r_1 \leq r
\leq r_2$, and by fitting a line to these points, in the least-squares
sense.

How can we compute the correlation dimension of a binary dataset $D$?
The probability $\prob{Z_D < r}$ can be computed
\[
\frac{1}{\abs{D}^2}\sum_{x \in D}\sum_{y \in D} I(\abs{x-y} < r),
\]
where $I(\abs{x-y} < r)$ is the indicator function having value $1$ if
$\abs{x-y} < r$, and value $0$ otherwise.  Computing the values
$\prob{Z_D < r}$ for all $r$ can thus be done trivially in time $O(N^2
K)$, where $N$ is the number of points in $D$ and $K$ is the number of
variables.  A sparse matrix representation yields to a running time of
$O(NM)$, where $M$ is the total number of 1's in the data: If point $i$
has $m_i$ 1's, then $\sum_i m_i = M$, and computing the all pairwise
distances takes time
\[
\sum_{i=1}^N\sum_{j=1}^N (m_i + m_j) = 2NM.
\]

If the number of points in a dataset is so large that quadratic
computation time in the number of points is too slow, we can take a
random subset $D_s$ from $D$ and estimate the probability $\prob{Z <
r}$ by
\[
\frac{1}{\abs{D}\abs{D_s}}\sum_{x \in D}\sum_{y \in D_s} I(\abs{x-y} < r)
\]
or by
\[
\frac{1}{\abs{D_s}^2}\sum_{x \in D_s}\sum_{y \in D_s} I(\abs{x-y} < r).
\]

\section{Properties of binary correlation dimension}
\label{sec:properties}
In this section we analyze the properties of the correlation dimension
$\cdR{D}{r_1, r_2}$ for binary datasets.  We show the following
results under some simplifying assumptions.  First, we prove that if
the original data has independent columns, then the correlation
dimension grows as the probabilities of the individual variables get
closer to $0.5$.  Second, we show that in the independent case
$\cdA{D}{\alpha,1-\alpha}$ grows as $\sqrt{K}$, where $K$ is the
number of attributes (columns) in the dataset.  Third, we prove that
if the variables are not independent, then the correlation dimension
is smaller than for a dataset with the same margins but independent
variables.

The analysis is not easy, and we need to make some simplifying
assumptions.  One complication is caused by the fact that the
definition of $\cdR{D}{r_1,r_2}$ involves computing the slope of a set
of points. However, note that $\mathcal{I}\pr{D, r_1, r_2, 1}$
contains only two points, and hence we have
\[
\cdR{D}{r_1,r_2, 1} = \frac{\log f(r_2) - \log f(r_1)} 
                         {\log r_2 - \log r_1}.
\]

Similarly, in the case of $\cdA{D}{\alpha_1,\alpha_2, 1}$ we have
$r_1$ and $r_2$ such that $\alpha_i = f(r_i)$, and hence
\[
\cdA{D}{r_1,r_2, 1} = \frac{\log \alpha_2 - \log \alpha_1} 
                         {\log r_2 - \log r_1}.
\]
Throughout this section we will assume that the parameter $N$ in
$\mathcal{I}\pr{D, r_1, r_2, N}$ is equal to $1$.

\vspace{0.2cm}
\begin{proposition}
\label{proposition:independent}
Assume that the dataset $D$ has $K$ independent variables, and that
the probability of the variable $i$ being 1 is $p_i$ for each $i$, and
let $q_i = 2 p_i (1-p_i)$.  Assuming that $K$ is large enough, we have
\[
\cdA{D}{\alpha,1-\alpha} \approx C(\alpha) \frac{\sum_i q_i}{\sqrt{\sum_i q_i (1-q_i)}},
\]
where $C(\alpha)$ is a constant depending only on $\alpha$.  In
particular, if all probabilities $p_i$ are equal to $p$, then for $q =
2 p (1-p)$ we have
\[
\cdA{D}{\alpha, 1-\alpha} = C(\alpha) \sqrt{\frac{Kq}{1-q}}.
\]
\end{proposition}
\vspace{0.2cm}

The proposition indicates that the correlation dimension is maximized
for variables as close to $0.5$ as possible.

\vspace{0.2cm}
\begin{corollary}
Assume the dataset $D$ has independent columns. The correlation
dimension $\cdA{D}{\alpha,1-\alpha}$ is maximized if the variables
have frequency $0.5$.
\end{corollary}
\vspace{0.2cm}

The proposition also tells that for a dataset with independent
identically distributed columns, the dimension grows as a square root
of the number of columns.

\begin{proof}[Proof of Proposition~\ref{proposition:independent}]
Recall that 
\[
\cdA{D}{\alpha,1-\alpha} = \frac{\log (1 - \alpha) - \log \alpha} 
                         {\log r_2 - \log r_1},
\]
where $r_1$ and $r_2$ are such that $\alpha = f(r_1)$ and
$1-\alpha = f(r_2)$.  The numerator is $\log((1-\alpha)/\alpha)$.
Assume that $K$ is large enough that we can estimate $f(r)$ by
$\prob{Z_D < r}$.

We next study the denominator $\log r_2 - \log r_1$.  We have to
analyze the distribution of the random variable $Z_D$, the $L_1$
distance between two randomly chosen points from $D$.  For simplicity,
we denote $Z_D$ by $Z$ in the sequel.  Let $Z_i$ be the indicator
variable having value 1 if two randomly chosen elements from $D$
disagree in variable $i$; then $Z = \sum_{i=1}^K Z_i$.

Denote by $q_i = \mean{Z_i}$ the probability that two randomly chosen
points from $D$ differ in coordinate $i$.  If $p_i$ is the probability
that variable $i$ in $D$ has value $1$, then $q_i = 2p_i(1 - p_i)$,
and it is easy to see that $q_i \leq 1/2$.

As $Z = \sum_{i=1}^K Z_i$, the variable $Z$ has a binomial
distribution.  For simplicity we use the normal approximation: $Z$ is
distributed as $N(\mu,\sigma)$, where $\mu = \sum_i q_i$ and $\sigma^2
= \sum_i q_i (1-q_i)$.  If $K$ is large enough, this approximation is
accurate.

By the symmetry of the normal distribution there is a constant $c$
such that $r_1 = \mu - c \sigma$ and $r_2 = \mu + c \sigma$.
Actually, $c$ is the inverse of the cumulative distribution function
of the normal distribution with parameters 0 and 1, i.e., $c =
\Phi^{-1}(\alpha) = \sqrt{2}\: \mathrm{erf}^{-1}(2 \alpha -1)$ The
denominator is
\[
   \log r_2 - \log r_1 = \log \frac{\mu + c \sigma}{\mu - c \sigma} 
   = \log\sum_{n = 0}^{\infty} \pr{\frac{2 c \sigma}{\mu}}^n.
\]
Dropping all but the two first terms and using the series for
logarithm we obtain that the numerate is
\[ 
   \log r_2 - \log r_1 \approx \frac{2 c \sigma}{\mu}.
\]
By setting 
\[
C(\alpha) =  \frac{\log ((1-\alpha)/\alpha)}{2c} = \frac{\log ((1-\alpha)/\alpha)}{2\sqrt{2}\: \mathrm{erf}^{-1}(2 \alpha -1)}
\]
we have the desired result.
\end{proof}

If $\alpha = 1/4$, then the constant $C(\alpha)$ in Proposition~\ref{proposition:independent}
is about $0.815$.

The correlation dimension has an interesting connection to the average
distance in randomly picked point pairs.

\vspace{0.2cm}
\begin{proposition}
\label{proposition:chernoff}
Assume that the dataset $D$ has $K$ independent variables, and that
the probability of variable $i$ being 1 is $p_i$.  Let $q_i = \sum_i 2
p_i (1-p_i)$. Let $\mu = \sum_i q_i$ be the average distance of two
randomly picked points.

Assume that we are given two constants $c_1$ and $c_2$ such that $0
\leq c_1 < c_2 \leq 1$. Then we can approximate the correlation
dimension as
\[
\cdR{D}{c_1\mu, c_2\mu} \approx C(c_1, c_2)\mu,
\]
where $C(c_1, c_2)$ depends only of $c_1$ and $c_2$.
\end{proposition}
\vspace{0.2cm}

Note that Proposition~\ref{proposition:chernoff} gives an
approximation for the quantity $\cd_R$, while
Proposition~\ref{proposition:independent} is about $\cd_A$; this,
however, is a superficial difference.  More important is the fact that
in Proposition~\ref{proposition:chernoff} we look at the case where
the bounds $r_1$ and $r_2$ are on the same side of the mean, whereas
the bounds corresponding to $\alpha$ and $1-\alpha$ from
Proposition~\ref{proposition:independent} are on the two sides of the
mean.  This implies that Proposition~\ref{proposition:chernoff} gives
a stronger bound: the dimension grows as a function of the mean $\mu$,
not as a function of $\mu/\sigma$.

\vspace{0.2cm}
\begin{example}
\label{exm:copydim}
Let $D$ be a dataset with $K$ dimensions, and consider the set $D'$
obtained by copying each variable in $D$ to $N$ new variables.  Then
$$\prob{Z_D < r} = \prob{Z_{D'} < Nr},$$ and hence 
$$\cdR{D}{r_1,
r_2} = \cdR{D'}{Nr_1, Nr_2}.$$
\end{example}
\vspace{0.2cm}

Given a dataset $D$ with $K$ columns, we denote by $\ind{D}$ a random
binary variable having $K$ independent components such that the
probability of $i$th component being $1$ is equal to the probability
of $i$th column of $D$ being $1$. Alternatively, $\ind{D}$ can be
considered as a dataset obtained by permuting each column of $D$
independently. We conjecture that the correlation dimension of $D$ is
always smaller than the correlation dimension of $\ind{D}$, given that
the original variables are all positively correlated.

\vspace{0.2cm}
\begin{conjecture}
\label{conjecture:positive}
Assume the marginal probability of all original variables are less
than $0.5$, and that all pairs of original variables are positively
correlated.  Then
\[
   \cdA{D}{\alpha,1-\alpha} \leq \cdA{\ind{D}}{\alpha,1-\alpha},
\]
i.e., the correlation dimension of the original data is not larger
than the correlation dimension of the data with each column permuted
randomly.
\end{conjecture}
\vspace{0.2cm}

Support for this conjecture is provided by the fact that the variance
$\var{Z_D}$ of the variable $Z_D$ can be shown to be no more than the
variance $\var{Z_{\ind{D}}}$; this does not, however, suffice for the
proof.  The intuition behind the above conjecture is similar to what
one observes in other types of definitions of dimension: if we
randomly permute each column of a dataset, we expect to see the rank
of the matrix to grow, and also explain an increase the number of PCA
components needed to explain, say, $90\%$ of the variance.  In the
experimental section we show the empirical evidence for
Conjecture~\ref{conjecture:positive}.
\section{Normalized correlation dimension}
\label{sec:inverse}

The definition of correlation dimension (Definition~\ref{def:cd}) is
based on the definition of correlation dimension for continuous data.
We have argued that the definition has some simple intuitive
properties: for a dataset with independent variables the dimension is
smaller if the variables are sparse, and the dimension shrinks if we
add structure to the data by making variables positively correlated.

However, the scale of the correlation dimension is not very intuitive:
the dimension of a dataset with $K$ independent variables is not $K$,
although this would be the most natural value.  The correlation
dimension gives much smaller values and hence we need some kind of
normalization.

We showed Section~\ref{sec:properties} that under some conditions
independent variables maximize the correlation dimension.  Informally,
we define the {\em normalized correlation dimension} of a dataset $D$
to be the number of variables that a dataset with independent
variables must have in order to have the same correlation dimension as
$D$ does.

More formally, let $\ind{H, p}$ be a dataset with $H$ independent
variables, each of which is equal to 1 with probability $p$.  From
Proposition 1 we have an explicit formula for $\cdA{\ind{H, p}}{\alpha, 1 - \alpha}$:
setting $q = 2 p (1-p)$ we have
\[
\cdA{\ind{H, p}}{\alpha, 1 - \alpha} \approx C(\alpha) \sqrt{\frac{Hq}{1-q}}.
\]
If the dataset would have the same marginal frequency, say $s$, for each
variable, the normalized correlation dimension of a dataset
$D$ could be defined to be the number $H$, such that
\[
\cdA{D}{\alpha,1-\alpha} \mbox{ and } \cdA{\ind{H, s}}{\alpha, 1 - \alpha}
\]
are as close to each other as possible.

The problem with this way of normalizing the dimension is that it
takes as the point of comparison a dataset where all the variables
have the same marginal frequency.  This is very far from being true in
real data.  Thus we modify the definition slightly.

We first find a value $s$ such that
\[
\cdA{\ind{K, s}}{\alpha, 1 - \alpha} = \cdA{\ind{D}}{\alpha, 1 - \alpha},
\]
i.e., a summary of the marginal frequencies of the columns of $D$: $s$
is the frequency that variables of an independent dataset should have
in order that it has the same correlation dimension as $D$ has when
the columns of $D$ have been randomized.  We define the \emph{normalized correlation dimension},
denoted by $\ncdA{D}{\alpha, 1 - \alpha}$,
to be an integer $H$ such that
\[
\cdA{\ind{H, s}}{\alpha, 1 - \alpha} = \cdA{D}{\alpha, 1 - \alpha}.
\]
Proposition~\ref{proposition:independent} implies the following statement.

\vspace{0.2cm}
\begin{proposition}
\label{proposition:normalized}
Given a dataset $D$ with $K$ columns, the dimension
$\ncdA{D}{\alpha, 1 - \alpha}$ can be approximated by
\[
\ncdA{D}{\alpha, 1 - \alpha} \approx \pr{\frac{\cdA{D}{\alpha, 1 - \alpha}}{\cdA{\ind{D}}{\alpha, 1 - \alpha}}}^2K.
\]
\end{proposition}
\vspace{0.2cm}

For examples, see the beginning of the next section.

\section{Experimental results}
\label{sec:experiments}

In this section we describe our experimental results.  We first
describe some results on synthetic data, and then discuss real datasets
and compare the normalized correlation dimension against PCA.

Unless otherwise mentioned, the dimension used in our experiments was
$\cdA{D}{\alpha_1, \alpha_2, N}$ such that $\alpha_1 = 1/4$, $\alpha_1 = 3/4$, and $N = 50$.

\subsection{Synthetic datasets}
\label{sec:expgenerated}
In this section we provide empirical evidence to support the analysis
in Sections~\ref{sec:properties}~and~\ref{sec:inverse}.  In the first
experiment we generated $100$ datasets with $K$ independent columns
and random margins $p_i$.  For each dataset, the margins $p_i$ were
randomly picked by first picking $p_{\text{max}}$ uniformly at random
from $\spr{0, 1}$.  Then, the probability $p_i$ was picked uniformly
from $\spr{0, p_{\text{max}}}$; this method results in datasets with
different densities.  The box plot in Figure~\ref{fig:generated_box}
shows that the normalized dimension is very close to $K$, the number
of variables in the data.  This shows that for independent data the
normalized correlation dimension is equal to the number of variables,
and that the sparsity of the data does not influence the results.

\begin{figure}
\centering
\includegraphics[width=5cm]{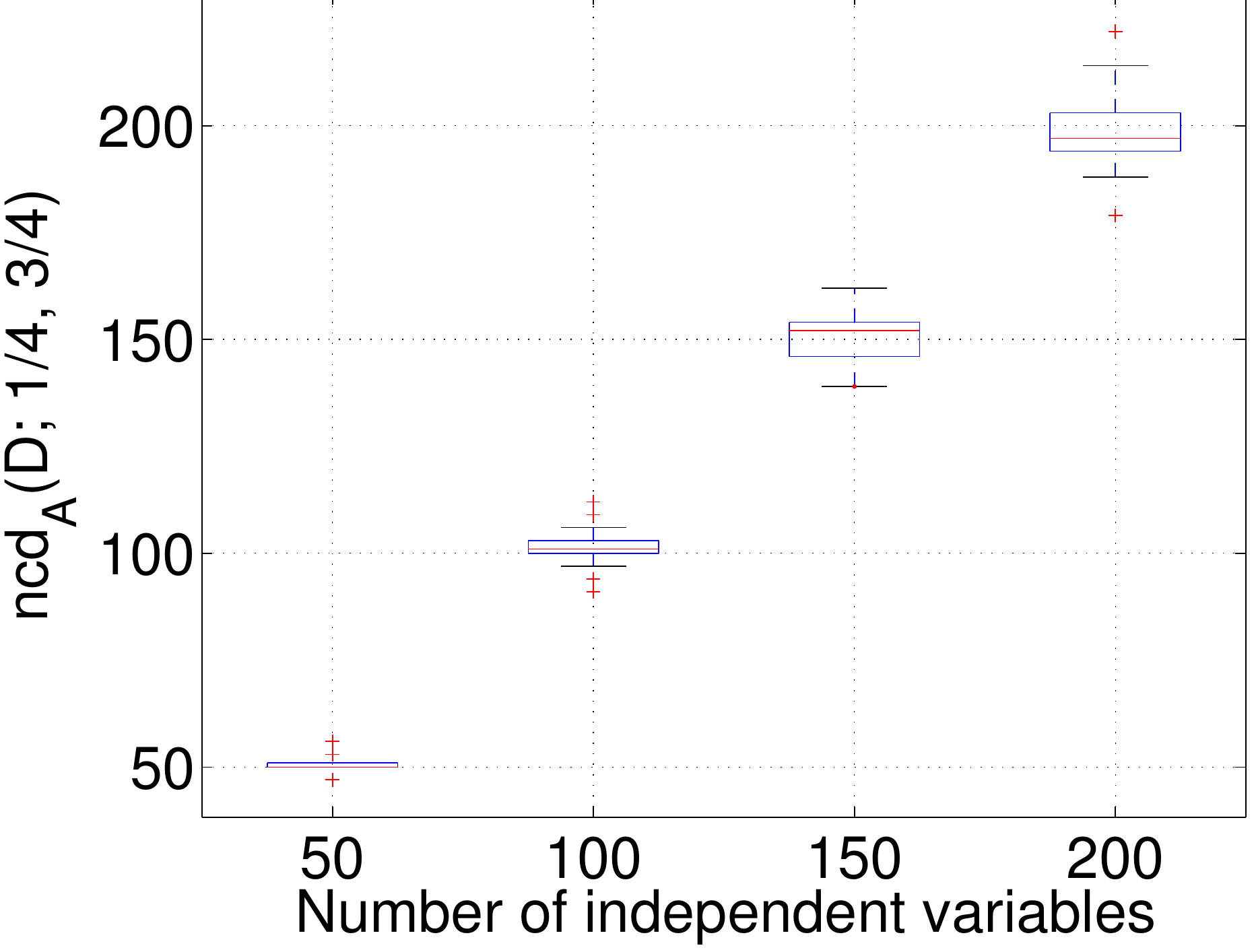}
\caption{Normalized correlation dimension for data having $K$
independent dimensions for $K \in \{50,100,150,200\}$.}
\label{fig:generated_box}
\end{figure}

Next we tested Proposition~\ref{proposition:independent} with
synthetic data.  We generated $100$ datasets having independent
columns and random margins, generated as described above.
Figure~\ref{fig:generated_a} shows the correlation dimension as a
function of $\mu/\sigma$, where $\mu = \mean{Z_{D}}$ and $\sigma^2 =
\var{Z_D}$.  The figure shows the behavior predicted by
Proposition~\ref{proposition:independent}: the normalized fractal
dimension is a linear function of $\mu/\sigma$, and the slope is very
close to $C(1/4)=0.815$.

\begin{figure}[htb]
\centering
\includegraphics[width=5cm]{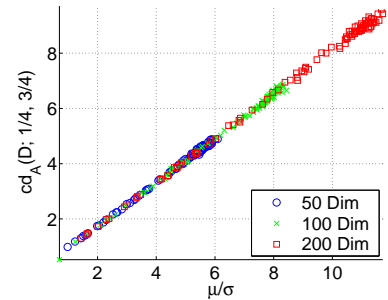}
\caption{Correlation dimension as a function of $\mu/\sigma$ for data
with independent columns (see
Proposition~\ref{proposition:independent}). The $y$-axis is
$\cdA{D}{1/4, 3/4}$ and the $x$-axis is $\mu/\sigma$, where $\mu =
\mean{Z_D}$ and $\sigma^2 = \var{Z_D}$. The slope of the line is about
$C(1/4) = 0.815$.}
\label{fig:generated_a}
\end{figure}

The theoretical section analyzes only the simplest form of the correlation
dimension, that is, the case where $N = 1$. We tested how the dimension
behaves for different $N$. In order to do that, we used generated datasets
from the previous experiments and plotted $\cdA{D}{1/4, 3/4, 50}$
against $\cdA{D}{1/4, 3/4, 1}$. We see from Figure~\ref{fig:generated_cd_vs_cd}
that the correlation dimension has little dependency of $N$.

\begin{figure}[htb]
\centering
\includegraphics[width=5cm]{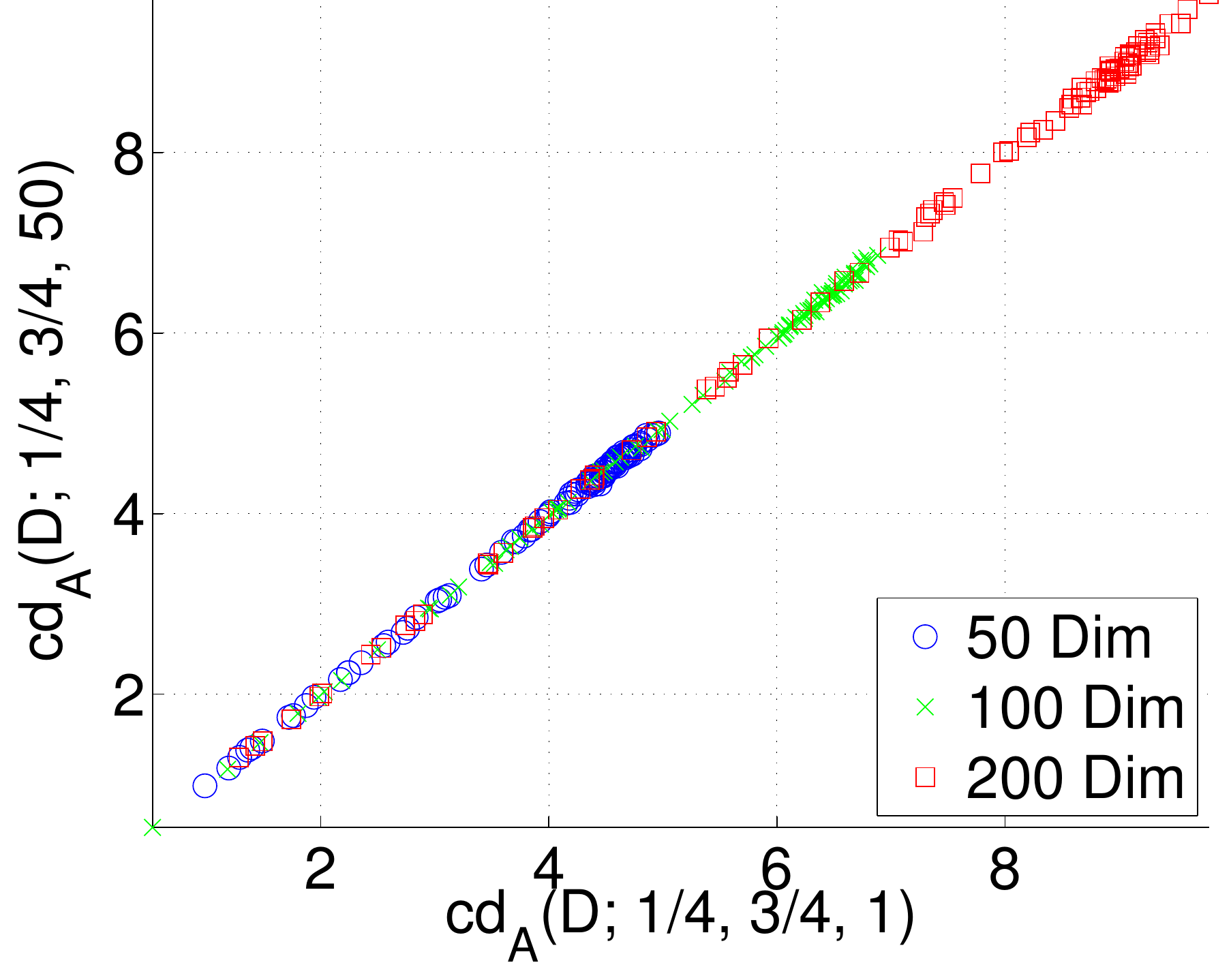}
\caption{Correlation dimension $\cdA{D}{1/4, 3/4, 50}$ as a function of
$\cdA{D}{1/4, 3/4, 1}$.}
\label{fig:generated_cd_vs_cd}
\end{figure}

Next we verified the quality of the approximation of
Proposition~\ref{proposition:chernoff}. We used the same data from the
previous experiment. Figure~\ref{fig:generated_b} shows the
correlation dimension against $\mu = \mean{Z_D}$, the average distance
of two random points. From the figure we see that
Proposition~\ref{proposition:chernoff} is partly supported: the
correlation dimension behaves as a linear function of $\mu$. However,
the slope becomes more gentle as the number of columns increases.

\begin{figure}[htb]
\centering
\includegraphics[width=5cm]{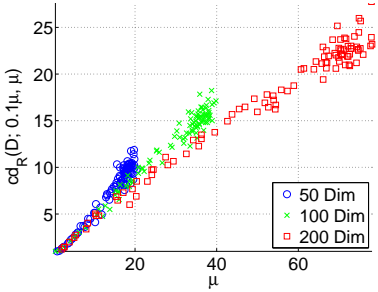}
\caption{Correlation dimension as a function of $\mu$ for data with
independent columns (see Proposition~\ref{proposition:chernoff}). The
$y$-axis is $\cdA{D}{1/4, 3/4}$ and the $x$-axis is $\mu =
\mean{Z_D}$, the average distance between two random points.}
\label{fig:generated_b}
\end{figure}

Our fifth experiment tested how positive correlation affects the
correlation dimension.  Conjecture~\ref{conjecture:positive} predicts
that positive correlation should decrease the correlation
dimension. We tested this conjecture by creating random datasets $D$
such that column $i$ depends on column $i-1$.  Let $X_i$ be variable
number $i$ in the generated dataset.  We generated data by a Markov
process between the variables:
\[
\prob{X_i = 1 \mid X_{i-1} = 0} = \prob{X_i = 0 \mid X_{i-1} = 1} = t_i
\]
and
\[
\prob{X_1 = 1} = \prob{X_1 = 0} = 0.5,
\]
where $X = \spr{X_1, \ldots, X_k}$ is the random element of $D$.

The reversal probabilities $t_i$ were randomly picked as follows: For
each dataset we picked uniformly a random number $t_{\text{max}}$ from
the interval $\spr{0, 1}$. We picked $t_i$ uniformly from the interval
$\spr{0, t_{\text{max}}}$.  Note that if the reversal probabilities
were $0.5$, then the dataset would have independent columns.  Denoting
$Z = Z_D$, we have
\[
\begin{split}
\prob{Z_i = 1 \mid Z_{i-1} = 0} & = \prob{Z_i = 0 \mid Z_{i-1} = 1} \\
& = 2t_i\pr{1-t_i}.
\end{split}
\]

A rough measure of the amount of correlation in the data is $t = \sum
2t_i\pr{1-t_i}$. Figure~\ref{fig:generated_c} shows the correlation
dimension as a function of the quantity $t$. We see that the datasets
with strong correlations tend to have small dimensions, as the theory
predicts.

\begin{figure}[htb]
\centering
\includegraphics[width=5cm]{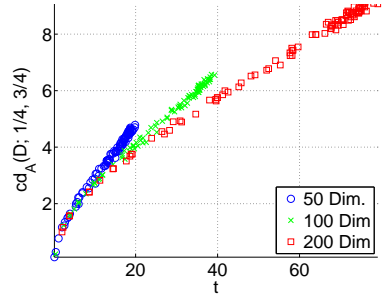}
\caption{Correlation dimension as a function of $t$, a rough measure
of correlation in a dataset. The $y$-axis is $\cdA{D}{1/4, 3/4}$ and
the $x$-axis is the quantity $t = \sum 2t_i\pr{1-t_i}$, where $t_i$ is
the reversal probability between columns $i$ and $i-1$.}
\label{fig:generated_c}
\end{figure}

Next, we go back to the first experiment to see whether the
normalized correlation dimension depends on the sparsity of data.
Note that sparse datasets have small $\mu = \mean{Z_D}$.
Figure~\ref{fig:generated_d} shows the normalized correlation
dimension as a function of $\mu$ for the datasets used in
Figure~\ref{fig:generated_box}.  We see that the normalized dimension
does not depend of sparsity, as expected.

\begin{figure}[htb]
\centering
\includegraphics[width=5cm]{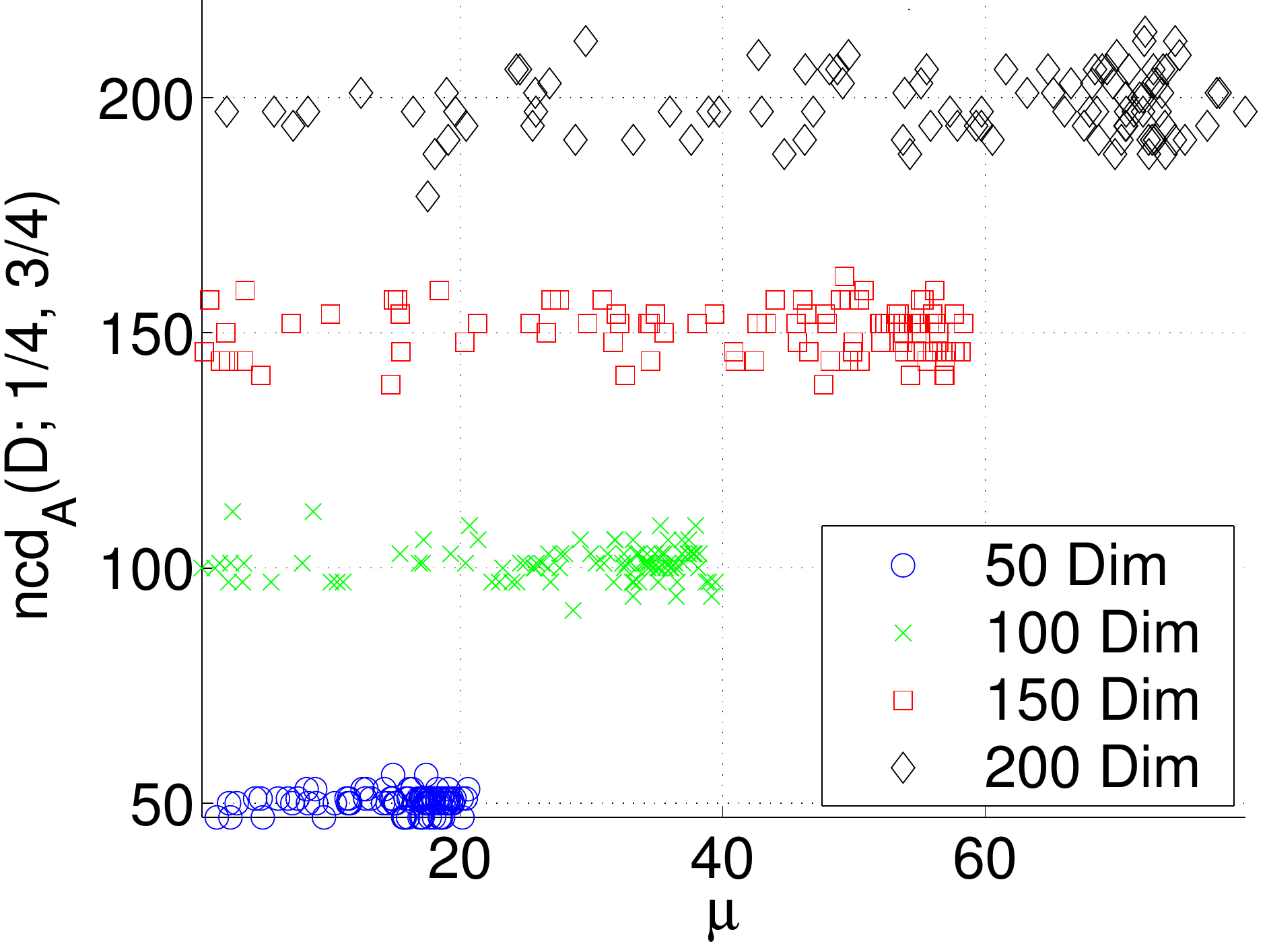}
\caption{Normalized correlation dimension as a function of $\mu$, the
average distance between two random points. The $x$-axis is $\mu =
\mean{Z_D}$ and the $y$-axis is $\ncdA{D}{1/4,3/4}$.}
\label{fig:generated_d}
\end{figure}

Finally, we tested Proposition~\ref{proposition:normalized} by
plotting the normalized dimension as a function of
$\frac{K\cdAnoarg{D}^2}{\cdAnoarg{\ind{D}}^2}.$ We used the
generated datasets from the previous experiment and from our fifth
experiment, as well.  Results given in
Figure~\ref{fig:generated_nfd_vs_base} reveal that the approximation
is good for the used datasets.

\begin{figure}[htb]
\centering
\includegraphics[width=4cm]{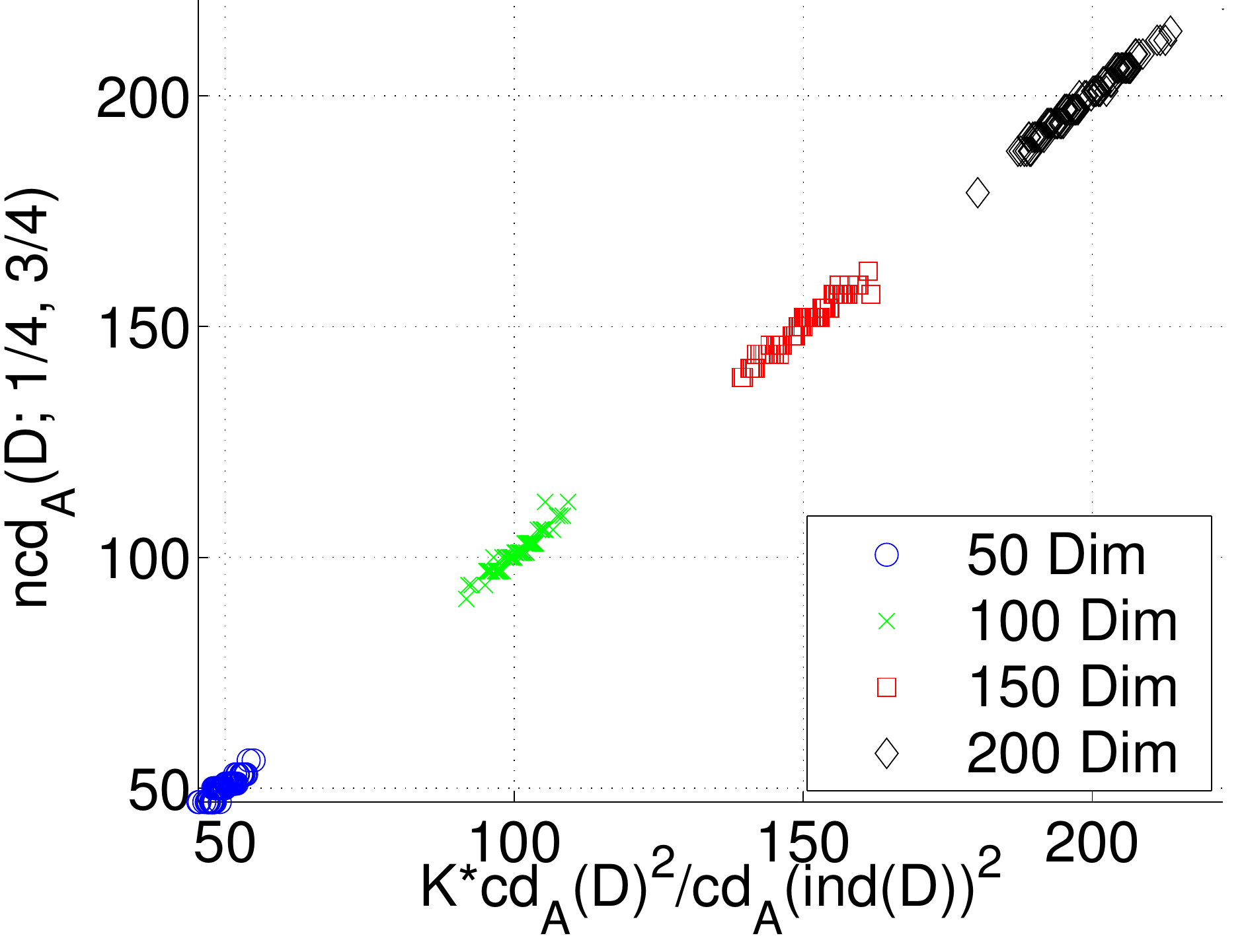}
\includegraphics[width=4cm]{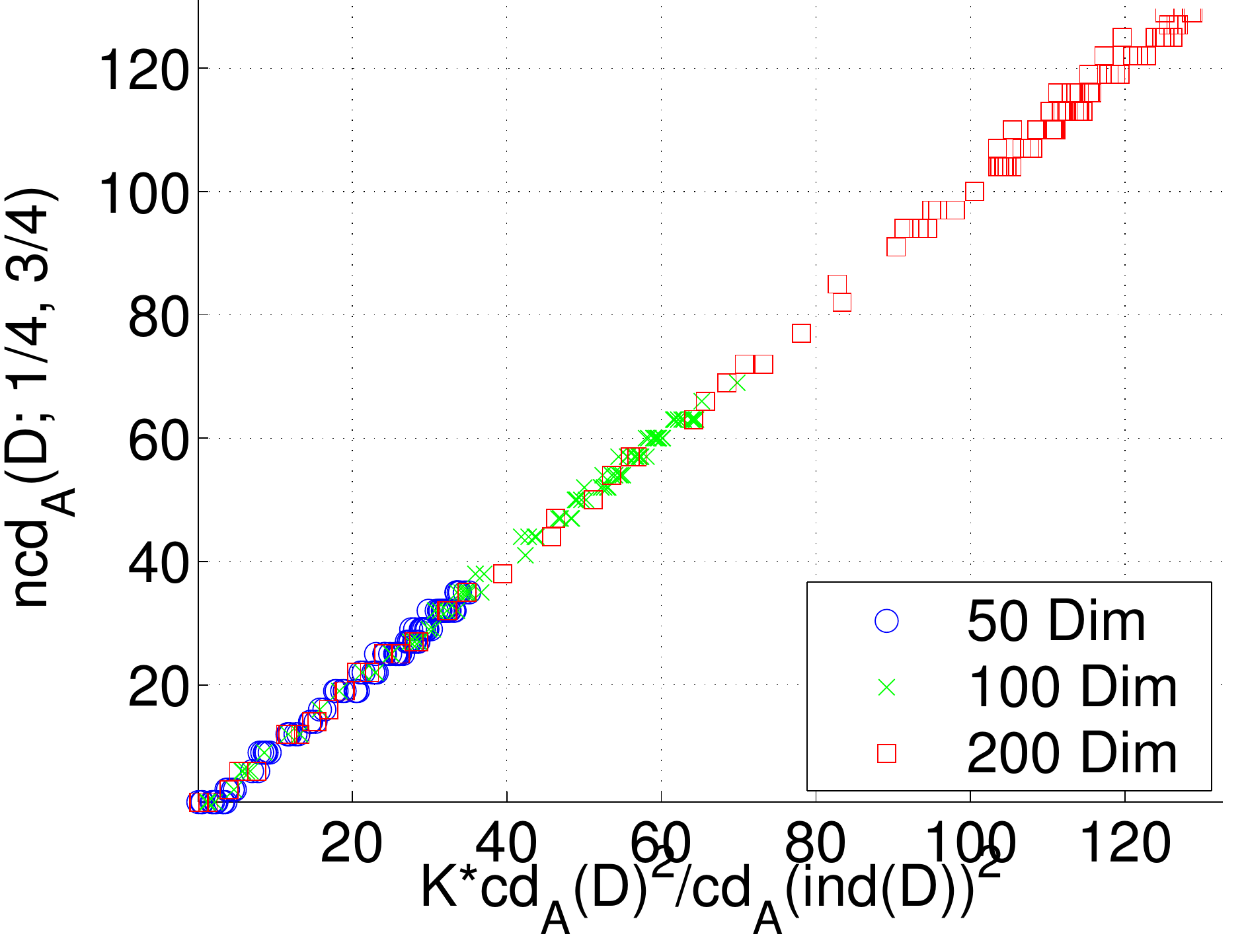}
\caption{Normalized correlation dimension as a function of
$K\cdAnoarg{D}^2/\cdAnoarg{\ind{D}}^2$.  The top figure contains
datasets with independent columns and in the bottom figure adjacent
columns of the datasets depend on each other.}
\label{fig:generated_nfd_vs_base}
\end{figure}

\subsection{Real-world datasets}
In this section we investigate how our dimensions behave with $9$
real-world datasets: \dtname{Accidents}, \dtname{Courses},
\dtname{Kosarak}, \dtname{Paleo}, \dtname{POS}, \dtname{Retail},
\dtname{WebView-1}, \dtname{WebView-2} and \dtname{20 Newsgroups}.
The basic information about the datasets is summarized in
Table~\ref{tab:basic}.

\begin{table}[ht!]
\centering
\caption{The basic statistics of the datasets. The column $K$
corresponds to the the number of columns and the column $N$ to the
number of rows. The last column is the density of 1's in percentages.}
\label{tab:basic}
\begin{tabular}{rrrrr}
\toprule
Data & $K$ & $N$ & \# of 1s & Dens. \\
\midrule
\dtname{Accidents} & $469$ & $340\,183$ & $11\,500\,870$ & $7.21$ \\
\dtname{Courses} & $5\,021$ & $2\,405$ & $64\,743$ & $0.54$ \\
\dtname{Kosarak} & $41\,271$ & $990\,002$ & $8\,019\,015$ & $0.02$ \\
\dtname{Paleo} & $139$ & $501$ & $3\,537$ & $5.08$ \\
\dtname{POS} & $1\,657$ & $515\,597$ & $3\,367\,020$ & $0.39$ \\
\dtname{Retail} & $16\,470$ & $88\,162$ & $908\,576$ & $0.06$ \\
\dtname{WebView-1} & $497$ & $59\,602$ & $149\,639$ & $0.51$ \\
\dtname{WebView-2} & $3\,340$ & $77\,512$ & $358\,278$ & $0.14$ \\
\bottomrule
\end{tabular}
\end{table}

The datasets are as follows. \dtname{20
Newsgroups}\footnote{\url{http://people.csail.mit.edu/jrennie/20Newsgroups/}}
is a collection of approximately $20\,000$ newsgroup documents across 20
different newsgroups~\cite{lang95newsweeder}.  Data in
\dtname{Accidents}\footnote{\url{http://fimi.cs.helsinki.fi/data/accidents.dat.gz}}
were obtained from the Belgian ``Analysis Form for Traffic Accidents''
forms that is filled out by a police officer for each traffic accident
that occurs with injured or deadly wounded casualties on a public road
in Belgium. In total, $340\,183$ traffic accident records are included
in the dataset~\cite{geurts03accidents}.  The datasets
\dtname{POS}\footnote{\url{http://www.ecn.purdue.edu/KDDCUP/data/BMS-POS.dat.gz}},
\dtname{WebView-1}\footnote{\url{http://www.ecn.purdue.edu/KDDCUP/data/BMS-WebView-1.dat.gz}}
and
\dtname{WebView-2}\footnote{\url{http://www.ecn.purdue.edu/KDDCUP/data/BMS-WebView-2.dat.gz}}
were contributed by Blue Martini Software as the KDD Cup 2000
data~\cite{kohavi00bms}. \dtname{POS} contains several years worth of
point-of-sale data from a large electronics retailer.
\dtname{WebView-1} and \dtname{WebView-2} contain several months worth
of click-stream data from two e-commerce web sites.
\dtname{Kosarak}\footnote{\url{http://fimi.cs.helsinki.fi/data/kosarak.dat.gz}}
consists of (anonymized) click-stream data of a Hungarian on-line news
portal.
\dtname{Retail}\footnote{\url{http://fimi.cs.helsinki.fi/data/retail.dat.gz}}
is a retail market basket data supplied by an anonymous Belgian retail
supermarket store~\cite{brijs99retail}.  The dataset
\dtname{Paleo}\footnote{NOW public release 030717 available
from~\cite{fortelius05now}.}  contains information of species fossils
found in specific paleontological sites in
Europe~\cite{fortelius05now}.  \dtname{Courses} is a student--course
dataset of courses completed by the Computer Science students of the
University of Helsinki.

We began our experiments by computing the correlation dimension
$\cdA{D}{1/4, 3/4}$ for each dataset. In order to do that, we needed
to estimate the probabilities $\prob{Z_D < r}$. Since some of the
datasets had a very large amount of rows (see Table~\ref{tab:basic}),
we estimate the probabilities $\prob{Z_D < r}$ by
\begin{equation}
\frac{1}{\abs{D}\abs{D_s}}\sum_{x \in D}\sum_{y \in D_s} I\pr{\abs{x-y} < r},
\label{eq:estimate}
\end{equation}
where $I\pr{\abs{x-y} < r}$ is $1$ if $\abs{x-y} < r$, and $0$
otherwise. The set $D_s$ was a random subset of $D$ containing $10\,000$
points. Since \dtname{Paleo} and \dtname{Courses} have small number of
rows, no sampling is used and $D_s$ was set to $D$ for these
datasets. The evaluation times are discussed in the end of the
section.

We also computed $\cdA{\ind{D}}{1/4, 3/4}$, the correlation dimension
for the datasets with the same column margins but independent columns.
Our goal was to use these numbers to provide empirical
evidence for the theoretical sections. To calculate the dimensions we
need to estimate the probabilities $\prob{Z_{\ind{D}} < r}$. The
estimation was done by generating $10\,000$ points from the distribution
of $Z_{\ind{D}}$.

The dimensions $\cdAnoarg{D}$ and $\cdAnoarg{\ind{D}}$ are given in
Table~\ref{tab:dims}. We see that the dimensions are very small. The
reason is that the datasets are quite sparse. 
We also observe that $\cdAnoarg{\ind{D}}$ is always larger than
$\cdAnoarg{D}$,  which suggests that there is at least some structure
in the datasets.

In addition, we used $\cdAnoarg{\ind{D}}$ to verify
Proposition~\ref{proposition:independent}. This was done by computing
$\mu/\sigma$, where $\mu = \mean{Z_{\ind{D}}}$ and $\sigma^2 =
\var{Z_{\ind{D}}}$. We also computed
\[
\hat{C}(1/4) = \cdA{\ind{D}}{1/4, 3/4}\frac{\sigma}{\mu}.
\]
Note that Proposition~\ref{proposition:independent} suggests that
$\hat{C}(1/4) \approx 0.8$. Table~\ref{tab:dims} shows us that this is
indeed the case.

\begin{table}[ht!]
\centering
\caption{Correlation dimensions of the datasets. In the second column,
$D' = \ind{D}$. The third column is
the fraction $\mu/\sigma$, where $\mu = \mean{Z_{D'}}$ and $\sigma^2 =
\var{Z_{D'}}$. The fourth column is an estimate of the coefficient
$C(1/4)$ obtained by dividing $\cdAnoarg{D'}$ with $\mu/\sigma$.}
\label{tab:dims}
\begin{tabular}{rrrrr}
\toprule
Data & $\cdAnoarg{D}$ & $\cdAnoarg{D'}$ & $\mu/\sigma$ & $\hat{C}\pr{1/4}$ \\
\midrule
\dtname{Accidents} & $3.79$ & $5.50$ & $6.67$ & $0.83$ \\
\dtname{Courses} & $1.56$ & $5.94$ & $7.29$ & $0.82$ \\
\dtname{Kosarak} & $0.96$ & $3.21$ & $3.96$ & $0.81$ \\
\dtname{Paleo} & $1.21$ & $3.20$ & $3.87$ & $0.83$ \\
\dtname{POS} & $1.14$ & $2.98$ & $3.62$ & $0.82$ \\
\dtname{Retail} & $1.33$ & $3.73$ & $4.49$ & $0.83$ \\
\dtname{WebView-1} & $1.27$ & $1.93$ & $2.26$ & $0.86$ \\
\dtname{WebView-2} & $1.01$ & $2.58$ & $3.05$ & $0.85$ \\
\bottomrule
\end{tabular}
\end{table}

We continued our experiments by calculating the normalized correlation
dimension $\ncdA{D}{1/4, 3/4}$.
For this we computed the probability $p$ such that 
\[
\cdA{\ind{K, p}}{\alpha, 1 - \alpha} = \cdA{\ind{D}}{\alpha, 1 - \alpha}
\]
using binary search.  Also, the normalized dimension itself was
computed by using binary search.  The normalized dimensions are given
in Table~\ref{tab:normdims}.

\begin{table}[ht!]
\centering
\caption{Normalized correlation dimensions of the datasets.}
\label{tab:normdims}
\begin{tabular}{rrrrr}
\toprule
Data & $K$ & $\ncd_A$ & $\frac{\ncdAnoarg{D}}{K}$ &  $\frac{K\cdAnoarg{D}^2}{\cdAnoarg{\ind{D}}^2}$ \\
\midrule
\dtname{Accidents} & $469$ & $220$ & $0.47$ & $222.91$ \\
\dtname{Courses} & $5\,021$ & $304$ & $0.06$ & $344.24$ \\
\dtname{Kosarak} & $41\,271$ & $2\,378$ & $0.06$ & $3\,684.78$ \\
\dtname{Paleo} & $139$ & $15$ & $0.11$ & $19.90$ \\
\dtname{POS} & $1\,657$ & $181$ & $0.11$ & $242.91$ \\
\dtname{Retail} & $16\,470$ & $1\,791$ & $0.11$ & $2\,107.52$ \\
\dtname{WebView-1} & $497$ & $190$ & $0.38$ & $214.33$ \\
\dtname{WebView-2} & $3\,340$ & $359$ & $0.11$ & $512.97$ \\
\bottomrule
\end{tabular}
\end{table}

Recall that the normalized correlation dimension of data $D$ indicates
how many variables a dataset $D'$ with independent columns should have
so that the distributional behavior of the pairwise distances between
points would be about the same in $D$ and $D'$.  Thus we note, for
example, that for the \dtname{Paleo} data the dimensionality is about
15, a fraction of $11\%$ of the number of columns in the original data.

The last column in Table~\ref{tab:normdims} is the estimate predicted
by Proposition~\ref{proposition:normalized}. Unlike with the synthetic
datasets (see Section~\ref{sec:expgenerated}), the estimate is poor in
some cases. A probable reason is that the examined datasets are
extremely sparse, and hence the techniques used to obtain
Proposition~\ref{proposition:normalized} are no longer accurate.
This is supported by the observation that \dtname{Accident}
has the best estimate and the largest density.

We also tested the accuracy of Proposition~\ref{proposition:normalized}
with \dtname{20 Newsgroups} dataset\footnote{The messages were converted
into bag-of-words representations and $200$ most informative variables were kept.}.
In Figure~\ref{fig:newses_nfd_vs_base} we plotted the normalized correlation
dimension as a function of the estimate. We see that the approximation
overestimates the dimension but the accuracy is better than in
Table~\ref{tab:normdims}.

\begin{figure}[htb]
\centering
\includegraphics[width=4.1cm]{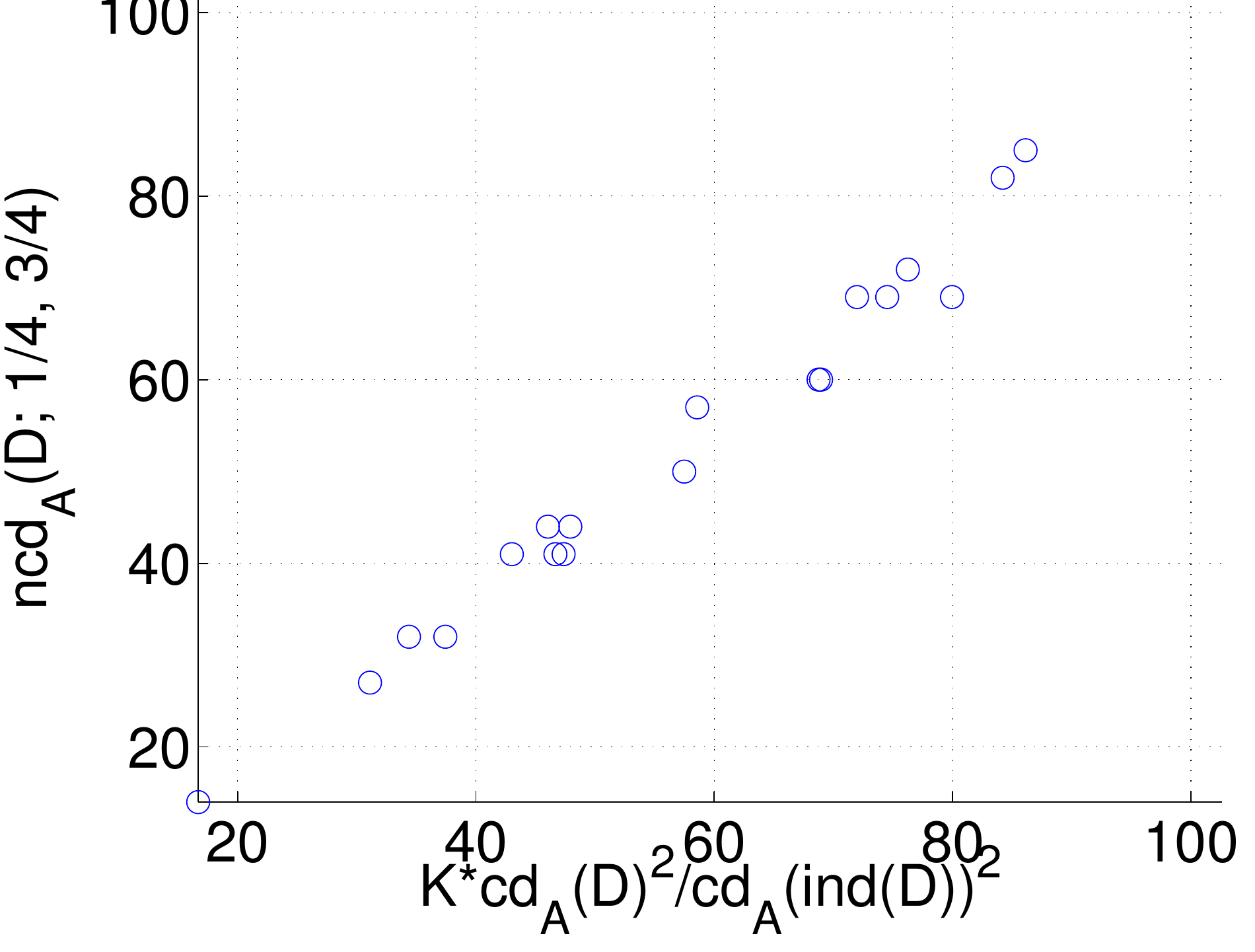}
\caption{Normalized correlation dimension as a function of
$K\cdAnoarg{D}^2/\cdAnoarg{\ind{D}}^2$.  Each point represents one
newsgroup in \dtname{20 Newsgroups} dataset.}
\label{fig:newses_nfd_vs_base}
\end{figure}

We will compare the normalized correlation dimensions against PCA in
the next subsection.

Next we studied the running times of the computation of the
correlation dimension.  Computing the distance of two binary vectors
can be done in $O(M)$ time, where $M$ is the number of 1's in the two
vectors.  Hence, estimating the probabilities using
Equation~\ref{eq:estimate} can be done in $O(\abs{D_s}L)$, where $L$
is the number of 1's in $D$.  We need also to fit the slope to get the
actual dimension, but the time needed for this operation is negligible
compared to the time needed for estimating the probabilities.  Note
that in our setup, the size of $D_s$ was fixed to $10\,000$ (except
for \dtname{Paleo} and \dtname{Courses}). Hence, the running time is
proportional to the number of 1's in a dataset. The running times are
given in Table~\ref{tab:times}.

\begin{table}[ht!]
\centering
\caption{The running times of the correlation dimension in seconds for
various datasets.  Time/\# of 1's: time in milliseconds divided by the
number of 1's in the data.}
\label{tab:times}
\begin{tabular}{rrrr}
\toprule
Data & \# of 1's & Time & Time/\# of 1's \\
\midrule
\dtname{Accidents} &     $11\,500\,870$ & $973$ & $0.085$ \\
\dtname{Courses}   &        $64\,743$ & $9$   & $0.141$ \\
\dtname{Paleo}     &        $3\,537$  & $0.1$ & $0.039$ \\
\dtname{Kosarak}   &      $8\,019\,015$ & $793$ & $0.099$ \\
\dtname{POS}       &      $3\,367\,020$ & $447$ & $0.133$ \\
\dtname{Retail}    &       $908\,576$ & $103$ & $0.113$ \\
\dtname{WebView-1} &       $149\,639$ & $17 $ & $0.114$ \\
\dtname{WebView-2} &       $358\,278$ & $40 $ & $0.112$ \\
\bottomrule
\end{tabular}
\end{table}

\subsection{Correlation Dimension vs.\ other methods}

There are different approaches for measuring the structure of a
dataset.  In this section we study how the normalized dimension
compares with other methods. Namely, we compared the normalized
fractal dimension against the PCA approach and the average correlation
coefficient.

We performed PCA to our datasets and computed the percentage of the
variance explained by the $M$ first PCA variables, where $M =
\ncdAnoarg{D}$.  Additionally, we calculated how many PCA components
are needed to explain $90\%$ of the variance.  The results are given in
Table~\ref{tab:pca}.  We observe that $\ncdAnoarg{D}$ PCA components
explain relatively large portion of the variance for \dtname{Accidents},
\dtname{POS}, and \dtname{WebView-1}, but explains less for
\dtname{Paleo} and \dtname{WebView-2}.

\begin{table}[ht!]
\centering
\caption{Normalized correlation dimensions versus PCA for various
datasets. The second column is the percentage of variance explained by
$\ncdAnoarg{D}$ variables and the third column is the number of
variables needed to explain $90\%$ of the variance.}
\label{tab:pca}
\begin{tabular}{rrrr}
\toprule
Data & $\ncdAnoarg{D}$ & PCA ($\%$) & $90\%$ PCA Dim. \\
\midrule
\dtname{Accidents} & $220$ & $99.83$ & $81$ \\
\dtname{Paleo} & $15$ & $48.50$ & $79$ \\
\dtname{POS} & $181$ & $84.48$ & $246$ \\
\dtname{WebView-1} & $190$ & $87.89$ & $208$ \\
\dtname{WebView-2} & $359$ & $59.73$ & $1\,394$ \\
\bottomrule
\end{tabular}
\end{table}

We next tested how robust the normalized correlation dimension is with
respect to the selection of variables.

Let us first explain the setup of our study.  Since especially PCA is
time-consuming, we created subsets of the data by taking randomly
$1000$ transactions\footnote{except for \dtname{Paleo} which had only
$501$ rows.}.  Let $\proj{M}{D}$ be the dataset obtained from $D$ by
selecting $M$ columns at random.  We used different numbers of
variables $M$ for different datasets. For each dataset $D$ we took
$50$ random subsets $\proj{M}{D}$ and use them for our analysis.

We first performed PCA to each $\proj{M}{D}$ and computed the number
of variables explaining $90\%$ of the variance.  We also computed the
average correlation coefficient for each dataset.  To be more precise,
let $c_{ij}$ be the correlation coefficient between columns $i$ and
$j$ in $\proj{M}{D}$.  We define the average correlation coefficient
to be
\[
\corr{D,M} = \frac{1}{M(M-1)}\sum_{i < j} \abs{c_{ij}}.
\]
Since structure in a dataset is seen as a small normalized fractal
dimension, we expect that $\ncdAnoarg{\proj{N}{D}}$ will correlate
positively with the PCA approach and negatively with the average
correlation coefficient $\corr{\proj{N}{D}}$.  The results are given
in Figure~\ref{fig:projections}.

\begin{figure}[htb]
\centering
\includegraphics[width=6cm]{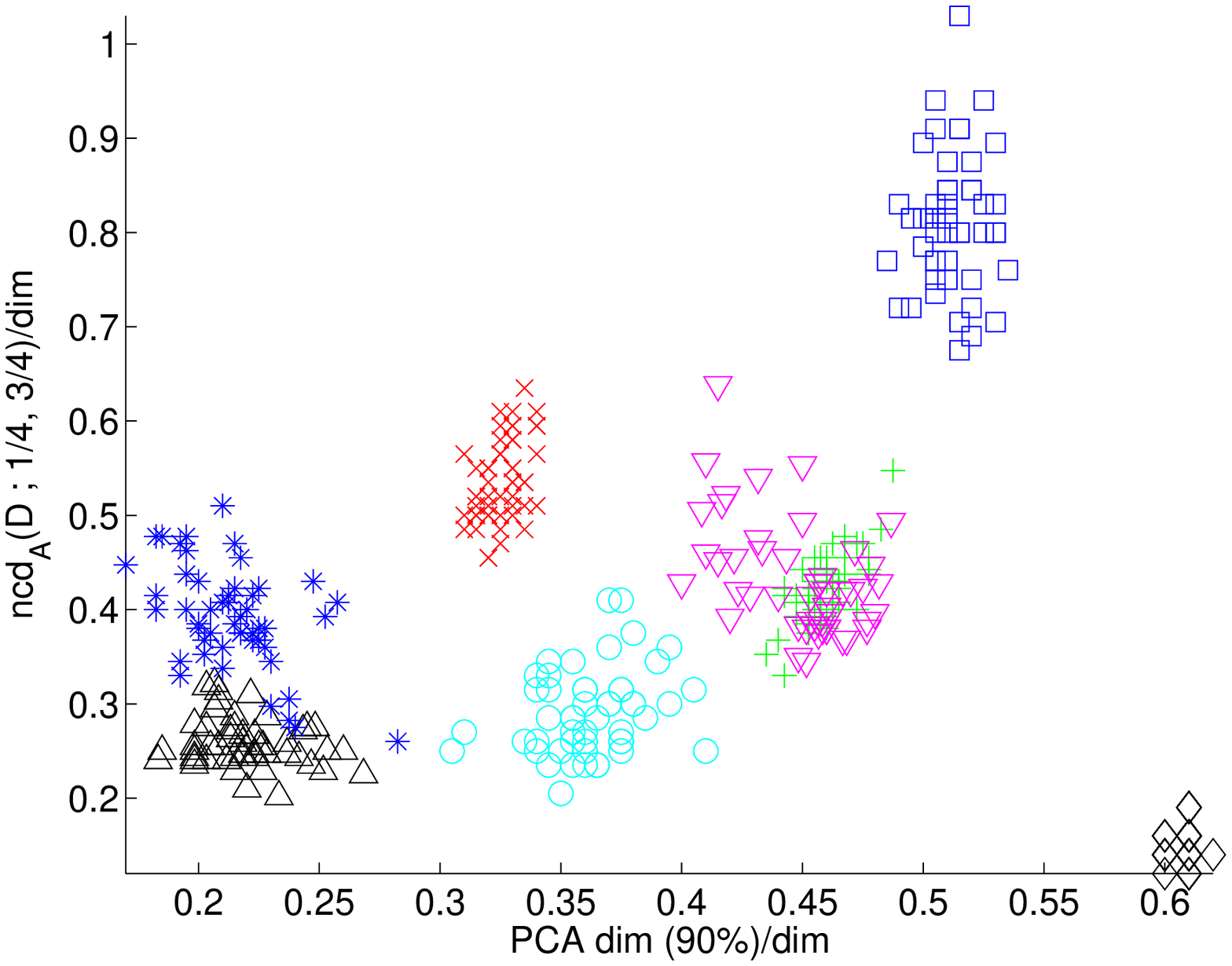}
\includegraphics[width=6cm]{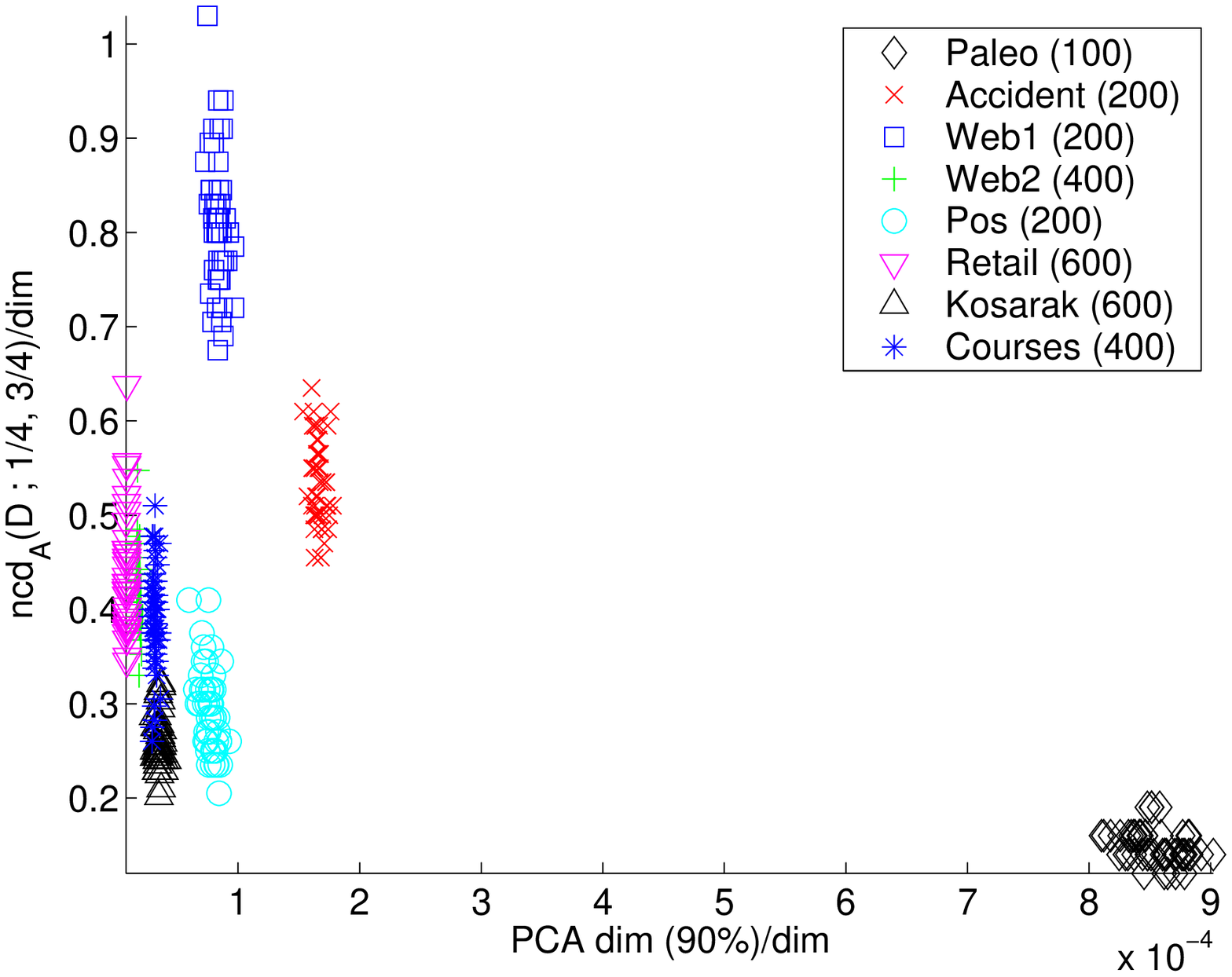}
\caption{Normalized correlation dimension for random subsets of the
data.  The $y$-axis is the normalized correlation dimension (divided
by the number of columns).  In the upper panel the $x$-axis is number
of PCA components needed to explain $90\%$ of the variance, divided by
the number of columns.  In the lower panel the $x$-axis is the average
correlation.  A single point represent one random subset of the
particular dataset.  The number of variables $M$ for the subset is
shown in parentheses in the legend.}
\label{fig:projections}
\end{figure}

We see from Figure~\ref{fig:projections} that there is a large degree
of dependency between these methods: The normalized dimension
correlates positively with PCA dimension and negatively with the
average correlation, as expected.  The most interesting behavior is
observed in the \dtname{Paleo} dataset.  We see that whereas PCA
dimension says that \dtname{Paleo} should have relatively high
dimension, the normalized dimension suggests a very small value.  The
average correlation agrees with the normalized dimension.  Also, we
know that \dtname{Paleo} has a very strong structure (by looking at
the data) so this suggests that the PCA approach overestimates the
intrinsic dimension for \dtname{Paleo}.  This behavior can perhaps be
partly explained also by considering the margins of the datasets.  The
margins of \dtname{Paleo} are relatively homogeneous whereas the
margins of the rest datasets are skewed.

We computed the correlation coefficients between the normalized
correlation dimension and the number of PCA components needed.  We
also computed the correlation for the normalized correlation dimension
and the average correlations.  These correlations coefficients were
computed for each dataset $D$ separately (recall that there were $50$
random subsets for each $D$).  Also, we calculated the correlations
for the case when all the datasets were considered simultaneously.  In
addition, since \dtname{Paleo} behaved like an outlier, we computed
the coefficients for the case where all datasets except \dtname{Paleo}
were present.  The results are given in Table~\ref{tab:corr_table} and
they support the conclusions we draw from
Figure~\ref{fig:projections}.

\begin{table}[ht!]
\centering
\caption{Correlations between normalized dimension against PCA
and average correlation.  Each row represents $50$ random
subsets of the particular dataset (see Figure~\ref{fig:projections}).
The second last row contains the correlations obtained by using the
subsets from all the datasets simultaneously.  The last row is similar
to the second last row except \dtname{Paleo} dataset was omitted.}
\label{tab:corr_table}
\begin{tabular}{rrr}
\toprule
& \multicolumn{2}{c}{$\ncdA{D}{1/4, 3/4}$ vs.} \\
\cmidrule{2-3}
Data & PCA ($90\%$) & $\corr{D}$ \\
\midrule
\dtname{Accident} & $0.44$ & $-0.23$ \\
\dtname{Courses} & $-0.51$ & $-0.01$ \\
\dtname{Kosarak} & $-0.21$ & $-0.02$ \\
\dtname{Paleo} & $0.10$ & $-0.31$ \\
\dtname{POS} & $0.27$ & $-0.54$ \\
\dtname{Retail} & $-0.48$ & $-0.18$ \\
\dtname{WebView-1} & $0.06$ & $-0.33$ \\
\dtname{WebView-2} & $0.70$ & $-0.49$ \\
\midrule
Total & $0.09$ & $-0.44$ \\
Total without \dtname{Paleo} & $0.60$ & $0.13$ \\
\bottomrule
\end{tabular}
\end{table}

\subsection{Correlation dimension for subgroups generated by clustering}

In this section we study how the correlation dimension of a
dataset is related to the dimensions of its subsets.  We consider the case
where the subsets are generated by clustering. The connection of the dimensions
of the clusters and the dataset itself is not trivial.

We first studied the subject empirically using the \dtname{Paleo}
dataset. There is a cluster structure in \dtname{Paleo}, and hence we
used $k$-means to find $3$ clusters and computed the dimensions for
these clusters. The dimensions are given in
Table~\ref{tab:paleoclust}.

\begin{table}[ht!]
\centering
\caption{Correlation dimension and normalized correlation dimension for \dtname{Paleo} data and its clusters. The clusters were obtained using the $k$-means algorithm.}
\label{tab:paleoclust}
\begin{tabular}{rrrr}
\toprule
Data & \# of rows & $\cdAnoarg{D}$ & $\ncdAnoarg{D}$ \\
\midrule
Cluster 1 & $51$ & $2.56$ & $37$ \\
Cluster 2 & $378$ & $1.60$ & $50$ \\
Cluster 3 & $72$ & $2.53$ & $46$ \\
Average & -- & $2.23$ & $44.33$ \\
Whole data & $501$ & $1.21$ & $15$ \\
\bottomrule
\end{tabular}
\end{table}

We also conducted experiments with \dtname{20 Newsgroups}.  First, we
calculated the normalized correlation dimension for each separate
newsgroup.  Then we created mixed datasets from $4$ newsgroups, one of
religious, one about computers, one recreational, and one science
newsgroup.  There were $240$ such datasets in total.  We computed the
dimensions for each mix and compare them to the average dimensions of
the newsgroups contained in the mixing.  The scatterplot of the
dimensions is given in Figure~\ref{fig:clusters}.

\begin{figure}[htb]
\centering
\includegraphics[width=4.1cm]{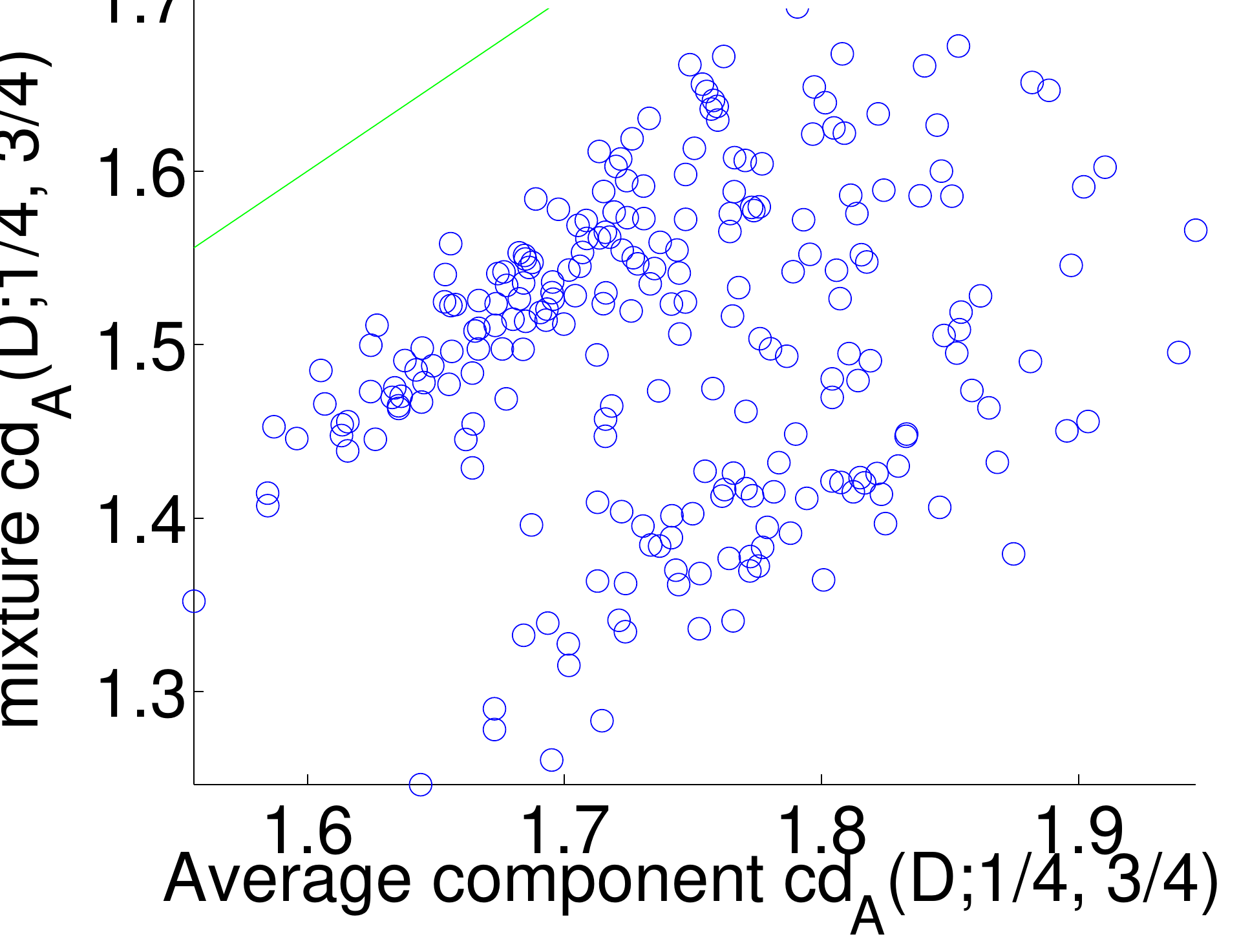}
\includegraphics[width=4.1cm]{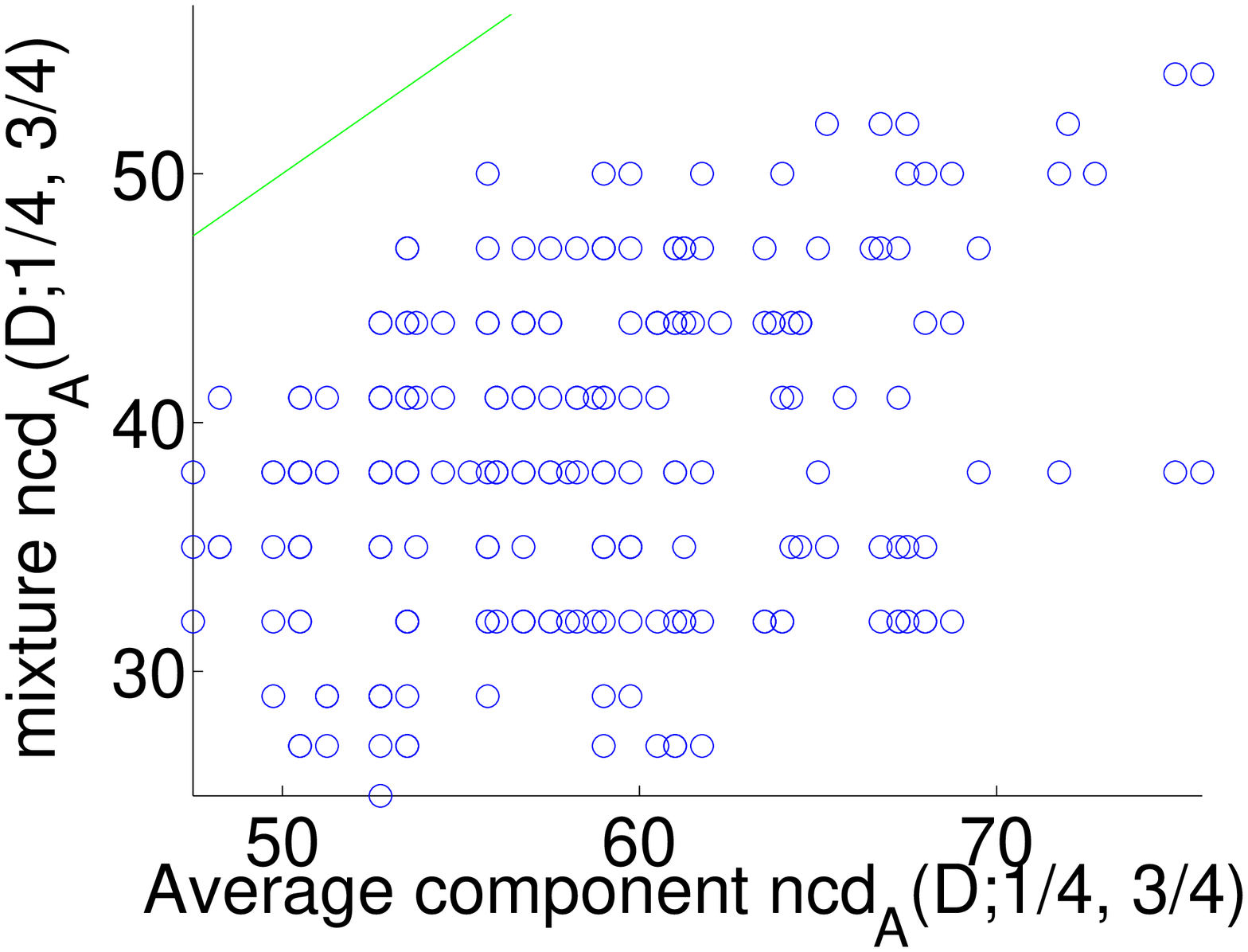}
\caption{Dimensions for cluster-structured data. Each
point represents a mixture of $4$ different newsgroups. The left figure
contains the correlation dimension and the right figure contains the
normalized correlation dimension. The $x$-axis is the average
dimension of the components used in a mixture and the $y$-axis is the
dimension of the mixture itself. }
\label{fig:clusters}
\end{figure}

From the results we see that for our datasets the clusters tend to
have higher dimensions than the whole dataset. We also see from
Figure~\ref{fig:clusters} that there is a positive correlation between
the dimension of a cluster and the dimension of the whole dataset.

\section{Related work}
\label{sec:related}

There has been a significant amount of work in defining the concept of
dimensionality in datasets.  Even though most of the methods can be
adapted to the case of binary data, they are not specifically tailored
for it.  For instance, many methods assume real-valued numbers and
they compute vectors/components that have negative or continuous
values that are difficult to interpret.  Such methods include, PCA,
SVD, and non-negative matrix factorization
(NMF)~\cite{jolliffe02pca,lee00nmf}.  Other methods such as
multinomial PCA (mPCA)~\cite{buntine03mpca}, and latent Dirichlet
allocation (LDA)~\cite{blei03lda} assume specific probabilistic models
of generating the data and the task is to discover latent components
in the data rather than reasoning about the intrinsic dimensionality
of the data.  Methods for exact and approximate decompositions of
binary matrices into binary matrices in Boolean semiring have also
been proposed~\cite{geerts04tiling,miettinen06dbp,monson95survey}, but
similarly to mPCA and LDA, they focus on finding components instead of
the intrinsic dimensionality.

The concept of fractal dimension has found many applications in the
database and data mining communities, such as, making nearest neighbor
computations more efficient~\cite{pagel00deflating}, speeding up feature selection
methods~\cite{jr00fast}, outlier detection~\cite{papadimitriou03loci},
and performing clustering tasks based on the local dimensionality of
the data points~\cite{gionis05dic}.

Many different notions of complexity of binary datasets have been
proposed and used in various contexts, for instance
VC-dimension~\cite{anthony97colt},
discrepancy~\cite{chazelle00discrepancy}, Kolmogorov
complexity~\cite{li97kolmogorov} and entropy-based
concepts~\cite{cover91it,palmerini04tdb}.  In some of the above cases,
such as Kolmogorov complexity and entropy methods, there is no direct
interpretation of the measures as a notion of dimensionality of the
data as they are measures of compressibility.  VC-dimension measures
the dimensionality of discrete data, but it is rather conservative as
a binary dataset having VC-dimension $d$ means that there are $d$
columns such that the projection of the dataset on those coordinates
results all possible bit vectors of length $d$.  Hence, VC-dimension
does not make any difference between datasets $\{0,1\}^d$ and $\{ x
\in \{0,1\}^K : \sum_{i=1}^K x_i \leq d\}$, although there is a great
difference when $d<<K$.  Furthermore, computing the VC-dimension of a
given dataset is a difficult problem~\cite{papadimitriou96vc}.

Related is also the work on random projections and dimensionality
reductions, such as in~\cite{achlioptas03projections}, but this line
of research has different goals than ours.  Finally, methods such as
multidimensional scaling (MDS)~\cite{kruskal64multidimensional} and
Isomap~\cite{tenembaum00reduction} focus on embedding the data (not
necessarily binary) in low-dimensional spaces with small distortion,
mainly for visualization purposes.
\section{Concluding remarks}
\label{section:conclusions}

We have given a definition of the effective dimension of a binary
dataset.  The definition is based on ideas from fractal dimensions: We
studied how the distribution of the distances between two random data
points from the dataset behaves, and fit a slope to the log-log set of
points.  We defined the notion of normalized correlation dimension.
It measures the number of dimensions of the appropriate density that a
dataset with independent variables should have to have the same
correlation dimension as the original dataset.

We studied the behavior of correlation dimension and normalized
correlation dimension, both theoretically and empirically.  Under
certain simplifying assumptions, we were able to prove approximations
for correlation dimension, and we verified these results using
synthetic data.

Our empirical results for real data show that different datasets have
clearly very different normalized correlation dimensions.  In general,
the normalized correlation dimension correlates with the number of PCA
components that are needed to explain $90\%$ of the variance in the
data, but there are also intriguing differences.

Traditionally, dimension means the degrees of freedom in the dataset.
One can consider a dataset embedded into a high-dimensional space by
some (smooth) embedding map.  Traditional methods such as PCA try to
negate this embedding. Fractal dimensions, however, are based on
different notion, the behavior of the volume of data as a function of
neighborhoods. This means that the methods in this paper do not
provide a mapping to a lower-dimensional space, and hence traditional
applications, such as feature reduction, are not (directly) possible.
However, our study shows that fractal dimensions have
promising properties and we believe that these dimensions are
important as such.

A fundamental difference between the normalized correlation dimension
and PCA is the following.  For a dataset with independent columns PCA
has no effect and selects the columns that have the highest variance
until some selected percentage of the variance is explained.  Thus,
the number of PCA components needed depends on the margins of the
columns.  On the other hand, the normalized correlation dimension is
always equal to the number of variables for data with independent
columns.

Obviously, several open problems remain.  It would be interesting to
have more general results about the theoretical behavior of the
normalized correlation dimension.  In the empirical side the study of
the correlation dimensions of the data and its subsets seems to be a
promising direction.

\bibliographystyle{latex8}
\bibliography{fractal}
\end{document}